%% file: arxiv.tex
\documentclass{article}
\usepackage[letterpaper, margin=1in]{geometry}
\usepackage[utf8]{inputenc}
\usepackage{authblk}
\usepackage{microtype}
\usepackage{graphicx}
\usepackage{subfigure}
\usepackage{booktabs}

\input{customized_command}

\title{Distributionally-Constrained Policy Optimization via Unbalanced Optimal Transport}
\author{Arash Givchi \thanks{arash.givchi@gmail.com}}
\author{Pei Wang}
\author{Junqi Wang}
\author{Patrick Shafto \thanks{patrick.shafto@gmail.com}}
\affil{Rutgers University-Newark}

\date{}

\begin{document}
\setlength\parindent{0pt}
\maketitle
\begin{abstract}
\noindent We consider constrained policy optimization in Reinforcement Learning, where the constraints are in form of  marginals on state visitations and global action executions. Given these distributions, we formulate policy optimization as unbalanced optimal transport over the space of occupancy measures. We propose a general purpose RL objective based on Bregman divergence and optimize it using Dykstra's algorithm. The approach admits an actor-critic algorithm for when the state or action space is large, and only samples from the marginals are available. We discuss applications of our approach and provide demonstrations to show the effectiveness of our algorithm.
\end{abstract}

\input{main-arxiv}

\bibliographystyle{apalike} 
\bibliography{references}

\appendix
\input{SI}

\end{document}

%% file: customized_command.tex
\usepackage{graphicx}
\usepackage{color,soul}
\usepackage{amsmath}
\usepackage{amsthm}
\usepackage{amstext}
\usepackage{mathrsfs}
\usepackage{amssymb}
\usepackage{bm}
\usepackage{pictexwd,dcpic}

\usepackage{tikz}
\usetikzlibrary{calc,intersections,through}
\tikzset{
    partial ellipse/.style args={#1:#2:#3}{
        insert path={+ (#1:#3) arc (#1:#2:#3)}
    }
}

\newtheorem{theorem}{Theorem}[section]

\theoremstyle{definition}

\newtheorem{proposition}[theorem]{Proposition}

\newtheorem{lemma}[theorem]{Lemma}
\newtheorem{corollary}[theorem]{Corollary}
\newtheorem{note}[theorem]{Note}

\newcommand{\R}{\mathbb{R}} 
\newcommand{\ent}{\mathcal{H}}
\newcommand{\K}{\text{KL}}
\newcommand{\E}{\mathbb{E}}
\newcommand{\prob}{\mathcal{M}}
\newcommand{\kl}{\text{KL}}

%% file: main-arxiv.tex
\section{Introduction}

In reinforcement learning (RL), policy optimization seeks an optimal decision making strategy, known as a policy \cite{bertsekas1995dynamic, williams1992simple, sutton2000policy}. 
Policies are typically optimized in terms of accumulated rewards with or without constraints on actions and/or states associated with an environment \cite{altman1999constrained}. 


Policy optimization has many challenges; perhaps the most basic is the constraint on flow of state-action visitations called \textit{occupancy measures}. 
Indeed, formulating RL as a linear programming problem, occupancy measures appear as an explicit constraint on the optimal policy \cite{puterman2014markov}. 
The constraint-based formulation suggests the possibility of implementing a broader set of objectives and constraints, such as entropy regularization \cite{peters2010relative,neu2017unified,nachum2020reinforcement} and cost-constrained MDPs \cite{altman1999constrained}.

Considering the reward function as negative cost of assigning an action to a state, we view RL as a stochastic assignment problem. 
We formulate policy optimization as an unbalanced optimal transport on the space of occupancy measures. 
Where Optimal Transport (OT) \cite{villani2008optimal} is the problem of adapting two distributions on possibly different spaces via a cost function, 
unbalanced OT relaxes the marginal constraints on  OT to arbitrary measures through penalty functions \cite{Liero_2017,chizat2017scaling}. 

We therefore define distributionally-constrained reinforcement learning as a problem of optimal transport.
Given baseline marginals over states and actions, policy optimization is unbalanced optimal transport adapting the state marginal to action marginal via the reward function.  
Built upon mathematical tools of OT, we generalize the RL objective to the summation of a Bergman divergence and any number 
of arbitrary lower-semicontinuous convex functions. We optimize this objective with \textit{Dykstra's algorithm} 
\cite{dykstra1983algorithm} which is a method of iterative projections onto general 
closed convex constraint sets. Under Fenchel duality, this algorithm allows decomposition 
of the objective into Bregman projections on the subsets corresponding to each function. 

As particular case, we can regularize over the state space distribution and/or the global action execution distribution of the desired occupancy measures. This formulation allows constraints on the policy optimization problem in terms of distributions on state visitations and/or action executions.
We propose an actor-critic algorithm with function approximation for large scale RL, 
for when we have access to samples from a baseline policy (off-policy sampling or imitation learning) and  samples from the constraint marginals.

The structure of the paper is as follows: 
In the Section~\ref{sec:note} we briefly present the preliminaries on (unbalanced) optimal transport and policy optimization in reinforcement learning. In Section~\ref{sec:Bregman} we introduce a general objective with Bregman divergence for policy optimization and provide Dykstra iterations as a general primal algorithm for optimizing this objective. 
Section~\ref{sec:DCPO} discusses  distributionally constraint policy optimization with unbalanced OT and its applications. In this section, we also provide an actor critic algorithm for large scale RL. We conclude with demonstrations of the distributional constraints in Section~\ref{sec:demo} and  discussion on related works in Section~\ref{sec:discussion}.

\section{Notation and Preliminaries}\label{sec:note}
For any finite set $\mathcal X$, let $\prob(\mathcal X)$ be the set of probability distributions on $\mathcal X$.
We denote 
\[\delta_{\mathcal X}(x) =
     \begin{cases}
      0 &\quad\text{if } x \in \mathcal X\\
      \infty &\quad\text{otherwise }\\
     \end{cases},
\]
as the indicator function on set $\mathcal X$. For $p\in \prob(\mathcal X)$, we define the \textit{entropy} map $\ent(p) =\E_{x\sim p}[\log p(x)-1]$, 
and denote the \textit{Kullback-Leibler (KL) divergence} between two positive functions $p, q$ by 
\[\K (p|q) =\sum_{x\in \mathcal X} p(x) \log \left(\frac{p(x)}{q(x)}\right) -p(x) +q(x).\]
If $p,q \in \mathcal M (\mathcal X)$, for a given convex function $\psi~:~\mathcal X\rightarrow \R $ with $\psi(1)=0$, we define $\psi$-divergence:
\[D_\psi(p|q)=\sum_{x} \psi\left(\frac{p(x)}{q(x)}\right)q(x).\]
In particular, for  $\psi(x)~ =x\log(x)$, $D_\psi(p|q)=\K(p|q)$.

We also use $\langle\cdot,\cdot\rangle$ as the natural inner product on $\mathcal X$.
Through out the paper by $\bm{1}_{\mathcal X}$ we denote a vector with elements one over set $\mathcal X$ or just $\bm{1}$ if the context is clear.

\subsection{Optimal Transport}
Given measures $a \in \prob(\mathcal X)$, $b \in \prob(\mathcal Y)$ on two sets $\mathcal X$ and $\mathcal Y$, with a cost function $C:\mathcal X\times \mathcal Y\rightarrow \R$, Kantorovich \textit{Optimal Transportation} is the problem of finding stochastic optimal assignment plan $\mu \in \prob(\mathcal X \times \mathcal Y)$:
\[
\underset{\mu}{\min}\, \mathbb E_{\mu}[C(x,y)] +\delta_{\{a\}}\left (\mu \bm{1}_{\mathcal Y}\right) + \delta_{\{b\}}\left (\mu^T \bm{1}_{\mathcal X}\right).
\]
When $\mathcal X=\mathcal Y$  and $C$ is derived from a metric on $\mathcal X$, this optimization defines a distance function on measure space $\prob(\mathcal X)$, called Wasserstein distance \cite{villani2008optimal}. 

\textit{Unbalanced Optimal Transport} replaces hard constraints $\delta_{\{a\}}$ and $\delta_{\{b\}}$, with penalty functions
\[
\underset{\mu}{\min} \, \mathbb E_{\mu}[C(x,y)] +\epsilon_1 D_{\psi_1}(\mu\bm{1}|a) + \epsilon_2D_{\psi_2}(\mu^T\bm{1}|b),
\]
where $\epsilon_1, \epsilon_2$ are positive scalars. This formulation also extends optimal transport to measures of arbitrary mass. As $\epsilon_1,\epsilon_2 \rightarrow \infty$, the unbalanced OT approaches Kantorovich OT problem \cite{Liero_2017, chizat2017scaling}.

To speed up the daunting computational costs of standard algorithms, an entropy term $\ent(\mu)$ is usually added to the (U)OT objective to apply scaling algorithms \cite{cuturi2013sinkhorn, chizat2017scaling} and \cite{p2020}.
When $\mathcal X$ and $\mathcal Y$ are large or continuous spaces,  we usually have access to samples from  $a,b$ instead of the actual measures. Stochastic approaches usually add a relative entropy $\K(\mu\mid a\otimes b)$,  instead of $\ent(\mu)$ in order to take advantage of the Fenchel dual of the (U)OT optimization and estimate the objective from samples out of $a,b$ \cite{aude2016stochastic, seguy2018largescale}.

\subsection{Reinforcement Learning}
\label{RL}
Consider a discounted MDP $(\mathcal S, \mathcal A, P, r, \gamma,p_0)$, with finite state space $\mathcal S$, finite action space $\mathcal A$,  transition model $P:\mathcal S\times \mathcal A\rightarrow  \prob(\mathcal S) $,   initial distribution  $p_0 \in \prob(\mathcal S)$, deterministic reward function $r:\mathcal S\times \mathcal A\rightarrow \R$ and discount factor $\gamma \in [0,1)$. Letting $\Pi$ be the set of stationary policies on the MDP, for any policy $\pi:\mathcal S \rightarrow \prob(\mathcal A)$, we denote $P^\pi:\mathcal S\rightarrow \prob(\mathcal S)$ to be the induced Markov chain on $S$. In policy optimization, the objective is
\begin{equation} \label{eq:rl_policy}
  \underset{\pi \in \Pi}\max\, \sum_{s,a}\rho^\pi(s) \pi(a|s) r(s,a),  
\end{equation}
where $\rho^\pi(s)=(1-\gamma)\sum_{t=0}^\infty \gamma^t \text{Pr}(s_t=s|\pi)$ is the discounted stationary distribution of $P^{\pi}$. For a policy $\pi$, we define its \textit{occupancy measure}  $\mu^\pi \in \prob(\mathcal S\times \mathcal A)$, as $\mu^\pi(s,a) = \rho^\pi(s) \pi(a|s)$. Let $ \Delta:=\{\mu^\pi: \pi \in \Pi\}$ be the set of occupancy measures of $\Pi$, the following lemma bridges the two spaces $\Pi$ and $\Delta$:
\begin{lemma}\cite{syed2008apprenticeship}[Theorem 2, Lemma2]
\label{syed}
 \begin{itemize}
 \item[(i)]   $\Delta=\{\mu\in \prob(\mathcal S\times \mathcal A): \sum_a \mu(s,a)=(1-\gamma) p_0(s)+\gamma\sum_{s',a'} P(s|s',a') \mu(s',a')\;\,  \forall s\}$
 \item[(ii)] $\pi^\mu(a|s):=\mu(s,a)/\sum_a \mu(s,a)$ is a bijection from $\Delta$ to  $\Pi$.
 \end{itemize}
\end{lemma}
So, by multiplying $\pi=\pi^{\mu}$ to both sides of the equation in (i), one can obtain
$\Delta = \{\mu : \mu\geq 0 \,,\, A^\mu \mu =b^\mu\}$ where $ A^\mu =\mathbb I  - \gamma {P_{\pi}}^T$ and $b^\mu =(1-\gamma)\pi p_0$. 
In the rest of paper, we may drop the superscripts $\mu$ and $\pi$, when the context is clear.
Rewriting the policy optimization objective \eqref{eq:rl_policy} in terms of $\mu$, we get
\begin{equation}\label{eq:rl_measure}
\underset{\mu \in \Delta}{\max}\, \mathbb E_\mu[r] = \underset{\mu}{\max}\, \mathbb E_\mu[r] + \delta_{\{b^\mu\}}(A^\mu \mu).
\end{equation}
Entropy-regularized version of objective \eqref{eq:rl_measure}, relative to a given baseline $\mu' \in \Delta$, is also studied \cite{peters2010relative,neu2017unified}:
\begin{equation}
\label{KLRL_shanon}
\underset{\mu \in \Delta}{\max}\, \mathbb E_\mu[r]-\epsilon\, \text{KL}(\mu|\mu') =  \underset{\mu \in \Delta}{\min}\, \text{KL}\left(\mu|\mu'e^{r/\epsilon_1}\right),
\end{equation}
where $\epsilon$ is the regularization coefficient.

By lemma \ref{syed}, one can decompose the regularization term in \eqref{KLRL_shanon} as
\begin{equation}
\label{RLKL}
\text{KL}(\mu|\mu') 
=\text{KL}\left(\sum_a \mu|\sum_a \mu'\right)+\mathbb E_{\rho^\mu}\left[\text{KL}\left(\pi| \pi'\right)\right],
\end{equation}
 with the first term  penalizing the shift on state distributions and the second penalty is over average shift of policies for every  state. Since the goal is to optimize for the best policy, one might consider only regularizing relative to $\pi'$  as in \cite{schulman2015trust, neu2017unified}
 \begin{equation}
\label{KLRL_conditional}
\underset{\mu \in \Delta}{\max}\, \mathbb E_\mu[r]-\epsilon\, \mathbb E_{\rho^\mu}\left[\text{KL}(\pi | \pi')\right].
\end{equation}
One can also regularize objective \eqref{eq:rl_measure} by $\mathcal H(\mu)$ as
\begin{equation}
\label{ERL}
\underset{\mu \in \Delta}{\max}\, \mathbb E_\mu[r]-\epsilon \mathcal H(\mu) 
= \underset{\mu \in \Delta}{\min}\, \text{KL}(\mu|e^{r/\epsilon_1}),
\end{equation}
\noindent
which encourages exploration and avoids premature convergence \cite{haarnoja2017reinforcement,schulman2018equivalence,ahmed2019understanding}.

\section{A General RL Objective with Bregman Divergence}
\label{sec:Bregman}

In this section, we propose a general RL objective based on Bregman divergence and optimize it using Dykstra’s algorithm.

Let $\Gamma$ be a strictly convex, smooth function on $\text{relint}(\text{dom}(\Gamma))$, the relative interior of its domain,  with convex conjugate $\Gamma^*$. 
For any $(\mu,\xi) \in \text{dom}(\Gamma)\times \text{int}(\text{dom}(\Gamma))$, we define the Bregman divergence 
$$D_\Gamma(\mu|\xi) =\Gamma(\mu)-\Gamma(\xi) -\langle\nabla\Gamma(\xi),\mu -\xi\rangle.$$
Given $\xi$, we consider the optimization
\begin{equation}
\label{General_OT}
\underset{\mu}\min \, D_{\Gamma}(\mu|\xi)+\sum_{i}^N \phi_i(\mu),
\end{equation}
where $\phi_i'$s are proper, lower-semicontinuous convex functions satisfying
\begin{equation}
\label{relint_condition}    
\cap_i^N\text{relint}(\text{dom}(\phi_i))\cap \text{relint}(\text{dom}(\Gamma))\neq \emptyset.
\end{equation}
Let $\text{dom}(\Gamma)$  be the simplex on $\mathcal S \times \mathcal A$, regularized RL algorithms in Section \ref{RL} can be observed as instances of optimization \eqref{General_OT}:
\begin{itemize}
\item Given a baseline $\mu'$, setting $\Gamma =\mathcal H$, $\phi_1(\mu) = \delta_{\{b^\mu\}}(A^\mu \mu)$,  $\xi =\mu'e^{r/\epsilon}$, recovers objective \eqref{KLRL_shanon}. 
\item Similarly, as discussed in \cite{neu2017unified}$,  \Gamma(\mu) = \sum_{s,a}\mu(s,a) \log {\mu(s,a)}/{\sum_a \mu(s,a)}$, $\phi_1 =\mathbb E_\mu[r/\epsilon]$, $\phi_2(\mu) = \delta_{\{b^\mu\}}(A^\mu \mu)$,  and  $\xi =\mu'$ recovers objective \eqref{KLRL_conditional}.
\item Further,  $\Gamma =\mathcal H$,  $\phi_1(\mu) = \delta_{\{b^\mu\}}(A^\mu \mu)$, and $\xi =e^{r/\epsilon_1}$, entropy-regularizes the occupancy measure in objective \eqref{ERL}.
\end{itemize}





The motivation behind using Bregman divergence is to generalize the KL divergence regularization usually used in RL algorithms.
Moreover, one may replace the Bergman divergence term in \eqref{General_OT} with a $\psi$-Divergence and attempt deriving similar arguments for the rest of the paper. 

\subsection{Dykstra's Algorithm}
In this section, we use Dykstra's algorithm \cite{dykstra1983algorithm} optimize objective \eqref{General_OT}. 
Dykstra is a method of iterative projections onto general closed convex constraint sets, 
which is well suited because the occupancy measure constraint is on a compact convex polytope $\Delta$. 

Defining the Proximal map of a convex function $\phi$, with respect to $D_\Gamma$, as 
\[
\text{Prox}_{\phi}^{D_\Gamma}(\mu) = \arg \underset{\tilde \mu}\min\, D_\Gamma(\tilde \mu|\mu)+\phi(\tilde \mu),
\]
for any $\mu \in \text{dom}(\Gamma)$, we present the following proposition which is the generalization of Dykstra algorithm in \cite{peyre2015entropic}:

\begin{proposition}[Dykstra's algorithm] \footnote{All proofs and derivations in this section are included in Appendix~\ref{apd:sec3}.}
\label{prop:Dykestra}
For iteration $l>0$,
\begin{equation}
\label{Dykestra}    
\begin{aligned}
&\mu^{(l)} =\text{Prox}_{\phi_{[l]_N}}^{D_\Gamma}\left( \nabla\Gamma^*\left(\nabla\Gamma(\mu^{(l-1)}) \right)+\nu^{(l-N)}\right)\\
&\nu^{(l)} =\nu^{(l-N)}+\nabla \Gamma(\mu^{(l-1)})-\nabla \Gamma(\mu^{(l)}),
\end{aligned}
\end{equation}
with 
\[
[l]_N =
     \begin{cases}
      N &\quad\text{if  } l\,\text{mod}\, N=0\\
      l\,\text{mod}\, N  &\quad\text{otherwise }\\
     \end{cases},
\]
converges to the solution of optimization \eqref{General_OT}, with $\mu^{(0)}=\xi$ and $\nu^{(i)}=\bm{0}$ for $-N+1\leq i \leq0$.
\end{proposition}
Intuitively, at step $l$,  algorithm \eqref{Dykestra} projects $\mu^{(l-1)}$ into the convex constraint set corresponding to the function $\phi_{[l]_N}$.

\begin{corollary} \label{cor:dykstra_kl}

Taking $\Gamma = \mathcal H$, the iteration \eqref{Dykestra} is
\begin{equation}\label{dykstra_kl}
\begin{aligned}
&\mu^{(l)} =\text{Prox}_{\phi_{[l]_N}}^{\text{KL}}\left( \mu^{(l-1)}\odot z^{(l-N)}\right)\\
&z^{(l)} =z^{(l-N)}\odot \frac{\mu^{(l-1)}}{\mu^{(l)}},
\end{aligned}
\end{equation}
where  $\odot, \frac{\cdot}{\cdot} $ are the element-wise  product and division, $\mu^{(0)}=\xi$ and $z^{(i)} =\bm{1}$, for  $-N+1\leq i \leq0$. 
\end{corollary}

\begin{note}
Given probability measures $a,b$, for $\Gamma =\ent$, $N=2$, $\phi_1(\mu) =\delta_{\{a\}}(\mu \bm{1})$, $\phi_2(\mu)=\delta_{\{b\}}(\mu^T \bm{1})$, 
optimization \eqref{General_OT} is entropic regularized optimal transport problem and algorithm \ref{dykstra_kl},
is the well known Sinkhorn-Knopp algorithm \cite{cuturi2013sinkhorn}. Similarly one can  apply \eqref{dykstra_kl} to solve the regularized UOT problem \cite{chizat2017scaling,chizat2019unbalanced}.
\end{note}


As  aforementioned  RL objectives in Section \ref{RL} can be viewed as instances of optimization \eqref{General_OT}, Dykstra's algorithm  can be used to optimize them. 
In particular, as the constraint $\phi_N(\mu) =\delta_{\{b^\mu\}}(A^\mu \mu)$ occurs in all of them,  each iteration of Dykstra requires 
\begin{equation}
\label{bregman_proj}
\text{Prox}_{\phi_N}^{D_\Gamma}(\mu) =\arg \underset{\tilde \mu \in \Delta}\min\, D_\Gamma(\tilde \mu|\mu),
\end{equation}
which is the Bregman projection of $\mu$ onto the space of occupancy measures $\Delta$.

Although $\mu$ (the measure from the previous step of Dykstra) does not necessarily lie inside $\Delta$,  step \eqref{bregman_proj} of Dykstra could  be seen as a Bregman divergence policy optimization  resulting in dual formulation over value functions (See details of dual optimization in Appendix~\ref{apd:sec4}). This dual formulation is similar to REPS algorithm \cite{peters2010relative}.  

In the next section we apply Dykstra to solve  unbalanced optimal transport on $\Delta$.
\section{Distributionally-Constrained Policy Optimization}
\label{sec:DCPO}

 
A natural constraint in policy optimization is to enforce a global action execution allotment and/or state visitation frequency. 
In particular, given a positive baseline  measure $\eta'$, with $\eta'(a)$ being a rough execution allotment of action $a$ over whole state space, for every $a\in \mathcal A$, 
we can consider $ D_{\psi_1}\left (\E_{\rho^\pi}[\pi]\mid \eta'\right)$ as a global penalty constraint of policy $\pi$ under its natural state distribution $\rho^\pi$.
Similarly, the penalty $D_{\psi_2}\left (\rho^\pi \mid \rho'\right)$  enforces a cautious or exploratory constraint 
on the policy behavior by avoiding or encouraging visitation of some states according to a given positive baseline measure $\rho'$ on~$\mathcal S$. 

So, given baseline measures $\rho'$ on $\mathcal S$ and $\eta'$ on $\mathcal A$, 
we define the distributionally-constrained policy optimization objective
\begin{equation}
\label{DCRL}
 \underset{\mu \in \Delta}\max
 \mathbb E_{\mu}[r] -\epsilon_1 D_{\psi_1}\left(\mu\bm{1}\mid \rho'\right)
-\epsilon_2 D_{\psi_2}\left(\mu^T\bm{1} \mid  \eta'\right).
\end{equation}


When $D_{\psi_1}=D_{\psi_2}=\K$, objective~\eqref{DCRL} looks similar to \eqref{KLRL_shanon} (considering expansion in \eqref{RLKL}), but they are different.
Because in \eqref{DCRL}, if $\rho' = \mu'\bm{1}$ and $\eta' =\mu'^T\bm{1}$, for some baseline $\mu' \in \Delta$, then the third term is $\K\left (\E_{\rho^\pi}[\pi]\mid \E_{\rho^{\pi'}}[\pi']\right)$ which is a global constraint on center of mass of $\pi$ over the whole state space, 
whereas $\E_{\rho^{\pi}}[\K(\pi\mid \pi')]$ in \eqref{RLKL} is a stronger constraint on closeness of policies on every single state.
The bottom line is that \eqref{DCRL} generally constrains the projected marginals of $\mu$ over $\mathcal S$ and $\mathcal A$, and \eqref{KLRL_shanon} constrains  $\mu$ element wise.

For regularization purposes in iterative policy optimization algorithm (e.g., using mirror descent), one natural choice of the state and action marginals is to take  $\rho'=\sum_a\mu_{k-1}$, $\eta'=\sum_s\mu_{k-1}$ at the $k$'th iteration. In the Appendix~\ref{apd:sec4}, we discuss the policy improvement and convergence of such an algorithm.
Another source for the marginals $\rho',\eta'$ is the empirical visitation of states and actions sampled out of an expert policy in imitation learning.

Formulation of Objective \eqref{DCRL} is in form of  unbalanced optimal transport on the space of occupancy measures.
So, for applying Dykstra algorithm, we can add an entropy term $\epsilon \ent(\mu)$ to transform \eqref{DCRL} into
\begin{align}
\label{DCRL2}
 \underset{\mu \in \Delta }\max -\K(\mu \mid \xi) - &\epsilon_1 D_{\psi_1}\left(\mu\bm{1}\mid  \rho'\right) 
-\epsilon_2 D_{\psi_2}\left(\mu^T\bm{1}\mid \eta'\right),
\end{align}
which means setting $N=3$, $\phi_1=\epsilon_1 D_{\psi_1}, \phi_2=\epsilon_2 D_{\psi_2}$, 
$\phi_3(\mu) = \delta_{b^{\mu}}(A^\mu \mu)$, $\xi = \exp(r/\epsilon)$ in Objective \eqref{General_OT} \footnote{Coefficients $\epsilon_1$  and $\epsilon_2$ in equation \eqref{DCRL2} are different from those in equation \eqref{DCRL}.}. Hence, the algorithm \eqref{dykstra_kl} 
can be applied with following proximal functions:
\begin{align}
  &\text{Prox}_{\phi_1}^{\text{KL}} (\mu)= \text{diag}\left( \frac{\text{Prox}_{\epsilon_1 D_{\psi_1}}^{\text{KL}}(\mu\bm{1})}{\mu\bm{1}}\right)\mu, \label{prox1}\\ 
  &\text{Prox}_{\phi_2}^{\text{KL}} (\mu)=\mu\, \text{diag}\left( \frac{\text{Prox}_{\epsilon_2 D_{\psi_2}}^{\text{KL}}(\mu^T\bm{1})}{\mu^T\bm{1}}\right),\label{prox2}\\
&\text{Prox}_{\phi_3}^{\text{KL}} (\mu)=\arg \underset{\tilde \mu \in \Delta}\min\, \text{KL}(\tilde \mu|\mu). \label{prox3}
\end{align} 

In general, for appropriate choices of $\phi_1$ and $\phi_2$ (e.g., $D_{\psi_1}=D_{\psi_2}=\K$)   the proximal operators in \eqref{prox1} and \eqref{prox2}  have closed form solutions. However, as discussed in the previous section, $\phi_3$ in \eqref{prox3} has no closed form solution\footnote{See detailed derivation for proximal operators in Appendix~\ref{apd:sec4}.}.
We can also consider other functions for $\phi_1, \phi_2$ in this scenario, for example, setting $\phi_2(\mu)=\delta_{\{\eta'\}}(\mu^T\bm{1})$, changes the problem into finding a policy $\pi$ which globally matches the distribution $\eta'$ under its natural state distribution $\rho^\pi$, i.e., $\mathbb E_{s \sim \rho^\pi}[\pi(a|s)] =\eta(a)$ for any $a \in \mathcal A$\footnote{As a  constraint, $\pi \rho^\pi$ would be a feasible solution, when  $\pi(a|s)=\eta'(a)$.}.

\noindent
In the next section we propose an actor-critic algorithm for large scale reinforcement learning.
\subsection{Large Scale RL}
\label{AC}
When $\mathcal S, \mathcal A$ are large, policy optimization via Dykstra is challenging because tabular updating of $\mu(s,a)$ is time consuming or sometimes even impossible. In addition, it requires  knowledge of reward function $r$ for the initial distribution $\xi$ and transition model $P$ for projection onto $\Delta$ in \eqref{prox3}. Model estimation usually requires  large number of state-action samples. Also, we might only have off-policy samples or in imitation learning scenarios,  only access the marginals $\rho',\eta'$ through observed samples by the expert.  In this section, we derive an off-policy optimization with function approximation to address these problems.

Replacing the last three terms of objective in \eqref{DCRL} with their convex conjugates by dual variables $u, v, Q$, we get\footnote{Appendix~\ref{apd:sec4_1} provides  derivations for all formulations in this section.}
\begin{align}
\label{ac1}
\max_\mu \min_{u,v,Q} \mathbb E_{\mu}\left[r-A^{{\mu}*} Q -\epsilon_1 u\bm{1_A}^T-\epsilon_2 \bm{1_S}v^T\right]
+\mathbb E_{b^\mu}[Q(s,a)] +\epsilon_1\mathbb E_{\rho'}[\psi_1^*({u(s))}]
+\epsilon_2 \mathbb E_{\eta'}[\psi_2^*(v(a))],
\end{align}
where ${A^\mu}^*$ is the convex conjugate (transpose) of $A^\mu$. 

Having access to samples from a baseline $\mu' \in \Delta$, both helps regularize the objective \eqref{ac1} into an easier problem to solve, and allows off-policy policy optimization or imitation learning \cite{nachum2020reinforcement}. 
Yet, by the discussion in Section~\ref{sec:DCPO}, regularizing with the term $D_{\psi}(\mu| \mu')$ in \eqref{ac1} might make marginal penalties redundant, in particular when $\rho' = \sum_a\mu'$ and $\eta' = \sum_s\mu'$. In the next two subsections we propose different approaches for each of these cases.

\subsubsection{$\rho' \neq \sum_a\mu'$ \textbf{or} $\eta' \neq \sum_s\mu'$}
Without the loss of generality, assume marginals $\rho' \neq \sum_a\mu'$ and $\eta' \neq \sum_s\mu'$.
In this case, regularizing \eqref{ac1} with $ D_{\psi}(\mu| \mu')$ and under Fenchel duality, we get the off-policy optimization objective
\begin{equation}
\label{ac2_pe}
\begin{split}
\underset{\pi}{\max}\min_{u,v,Q} &\mathbb E_{\mu'}\left[\psi^*\left(r+\gamma P^\pi Q-Q -\epsilon_1 u-\epsilon_2 v\right)(s,a)\right]
+(1-\gamma)\mathbb E_{\pi,p_0}[Q(s,a)] \\&+\epsilon_1\mathbb E_{\rho'}[\psi_1^*({u(s))}]
+\epsilon_2 \mathbb E_{\eta'}[\psi_2^*(v(a))],
\end{split}
\end{equation}
where $u(s,a):=u(s)$ and $v(s,a):=v(a)$. 
Now, the first term is based on expectation of baseline $\mu'$ and can be estimated from  off-policy samples.

In a special case, when $D_{\psi_1}=D_{\psi_2}=\K$ in objective \eqref{DCRL} and we take $D_{\psi}=\K$ as well, similar derivations yield
\begin{equation}
\label{ac3_pe}
\begin{split}
 \underset{\pi}{\max}\min_{u,v,Q}\mathcal  L(\pi,u,v,,Q):& =\mathbb \log \mathbb E_{\mu'}\left[\exp\left(r+\gamma P^\pi Q-Q -\epsilon_1 u-\epsilon_2 v\right)(s,a)\right]
+(1-\gamma)\mathbb E_{p_0,\pi}[Q(s,a)] \\&+\epsilon_1\mathbb \log \mathbb E_{\rho'}[\exp({u(s))}]
+\epsilon_2 \log \mathbb E_{\eta'}[\exp(v(a))].
\end{split}
\end{equation}

Now, gradients of $\mathcal L$ can now be computed. 
Letting $h^{\pi}_{u,v,Q}(s,a) := r(s,a)+\gamma P^\pi Q(s,a)-Q(s,a)-\epsilon_1 u(s)-\epsilon_2 v(a)$ and defining $\mathcal {F}_{p}\circ h(z) :=\mathcal {F}_{p}(h)(z) = \exp(h(z))/\mathbb E_{\tilde z\sim p}[\exp(h(\tilde z))]$ to be the softmax operator for any $h, z, p$, we have gradients of $\mathcal L$ with respect to $u,v,Q$:
\begin{equation}\label{20}
\begin{split}
\nabla_{u}\mathcal L=
-\epsilon_1 \mathbb E_{(s,a)\sim \mu'}[\mathcal {F}_{\mu'} \circ h^{\pi}_{u,v,Q}(s,a)\nabla u(s) ]
+\epsilon_1\mathbb E_{s\sim\rho'}[\mathcal F_{\rho'}\circ u(s)\nabla u(s)],
\end{split}
\end{equation}
\begin{equation}\label{21}
\begin{split}
\nabla_{v}\mathcal L=
-\epsilon_2 \mathbb E_{(s,a)\sim \mu'}[\mathcal {F}_{\mu'} \circ h^{\pi}_{u,v,Q}(s,a)\nabla v(a) ]
+\epsilon_2\mathbb E_{a\sim \eta'}[\mathcal F_{\eta'}\circ v(a)\nabla v(a)],
\end{split}
\end{equation}
\begin{equation}\label{22}
\begin{split}
\nabla_{Q}\mathcal L= \mathbb E_{(s,a)\sim \mu'}[\mathcal {F}_{\mu'}  \circ h^{\pi}_{u,v,Q}(s,a)
\left(\gamma P^{\pi} \nabla Q -\nabla Q \right){(s,a)}]
+(1-\gamma)\mathbb E_{\underset{a\sim\pi(\cdot|s)}{s\sim {p_0}}}[\nabla Q(s,a)].
\end{split}
\end{equation}
The gradient with respect to policy $\pi$ is
\begin{equation}\label{23}
\begin{split}
\nabla_{\pi}\mathcal L &=\gamma \mathbb E_{\underset{(s,a,s')\sim \mu'}{a'\sim\pi(\cdot|s')}}[\mathcal  {F}_{\mu'} \circ h^{\pi}_{u,v,Q}(s,a) Q(s',a')\nabla \log\pi(a'\mid s')]
+(1-\gamma)\mathbb E_{\underset{a\sim\pi(\cdot|s)}{s\sim {p_0}}}[Q(s,a)\nabla \log\pi(a\mid s)],
\end{split}
\end{equation}
where $u,v,Q, \pi$ can be approximated by some functions and one can apply gradient ascent on $\pi$ and gradient descent on $u,v,Q$.

\subsubsection{$\rho' = \sum_a\mu'$ \textbf{and} $\eta' = \sum_s\mu'$}
Following the approach in \cite{sutton2016emphatic,nachum2020reinforcement}, assuming $\pi$ is known, with the change of variable $\zeta(s,a):=\frac{\mu}{\mu'}(s,a)$,  we can rewrite  \eqref{ac1}  with importance sampling weights as a policy evaluation problem
\begin{equation}
\label{ac33}
\begin{split}
 \min_{u,v,Q} \max_\zeta \mathcal L(u,v,Q,\zeta;\pi):&=
 \mathbb E_{\mu'}\left[\zeta(s,a)\left(r+\gamma P^\pi Q -Q -\epsilon_1u-\epsilon_2v\right)(s,a)\right] +(1-\gamma)\mathbb E_{p_0,\pi}[Q(s,a)]\\&+\epsilon_1 \mathbb E_{\rho'}[\psi_1^*({u(s))}] +\epsilon_2 \mathbb E_{\eta'}[\psi_2^*(v(a))].
\end{split}
\end{equation}
The gradients with respect to $u,v,Q, \zeta$ are as follows:
\begin{equation}\label{25}
\begin{split}
\nabla_u \mathcal L(u,v,Q,\zeta;\pi)=&-\epsilon_1 \mathbb E_{\mu'}\left[\zeta(s,a)\nabla u(s)\right] 
+\epsilon_1 \mathbb E_{\rho'}[\nabla_u \psi_1^*({u(s))}],
\end{split}
\end{equation}
\begin{equation}\label{26}
\begin{split}
\nabla_v \mathcal L(u,v,Q,\zeta;\pi)=&-\epsilon_2 \mathbb E_{\mu'}\left[\zeta(s,a)\nabla v(a)\right]
+\epsilon_2 \mathbb E_{\eta'}[\nabla_v \psi_2^*({v(a))}],
\end{split}
\end{equation}
\begin{equation}\label{27}
\begin{split}
 \nabla_Q \mathcal L(u,v,Q,\zeta;\pi)= \mathbb E_{\mu'}\left[\left(\zeta+\gamma P^\pi \nabla Q -\nabla Q\right)(s,a) \right]+(1-\gamma)\mathbb E_{p_0,\pi}[\nabla Q(s,a)],
\end{split}
\end{equation}
\begin{equation}\label{28}
\begin{split}
 \nabla_{\zeta} \mathcal L(u,v,Q,\zeta;\pi)=
 \mathbb E_{\mu'}\left[\nabla\zeta(s,a)h^{\pi}_{u,v,Q}(s,a)\right].
\end{split}
\end{equation}
 Wrapping $\max_{\pi}$ around \eqref{ac33} gives the off-policy optimization. Given optimized $Q,\zeta$, the gradient with respect to $\pi$ is
\begin{equation}\label{29}
\begin{split}
 &\nabla_{\pi} \mathcal{L}(u,v,Q,\zeta,\pi)=\mathbb E_{(s,a)\sim \mu'}\left[\zeta(s,a)Q(s,a)\nabla\log\pi(a|s)\right] 
\end{split}
\end{equation}
\section{Demonstrations}
\label{sec:demo}
In this section, we demonstrate the effectiveness of distributionally-constrained policy optimization with Dykstra. The purpose of our experiments is to answer  how  distributional penalization on  $\rho'$ and $\eta'$ affect the behavior of the policy and study the Dykstra's empirical rate of convergence .

We look into these questions on a grid world with its optimal policy out of \eqref{ERL} shown in Fig.~\ref{fig:fig_gw}.
Due to the simplicity, this environment is suitable for studying the effect of  distributional constraints on the policy behavior.
For the sake of focusing on the effects of distributional constraints, we set the coefficient of the entropy term fairly low ($\epsilon=.01$) in optimizations \eqref{ERL} and \eqref{DCRL2}.~\footnote{Appendix~\ref{apd:demo} provides the numerical settings in implementation of Dykstra.}

\begin{figure}[h!]
\centering
  \includegraphics[scale=0.55]{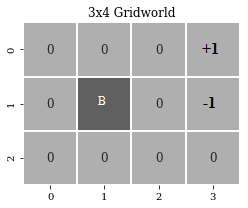}
  \hspace{0.4in}
  \includegraphics[scale=0.25]{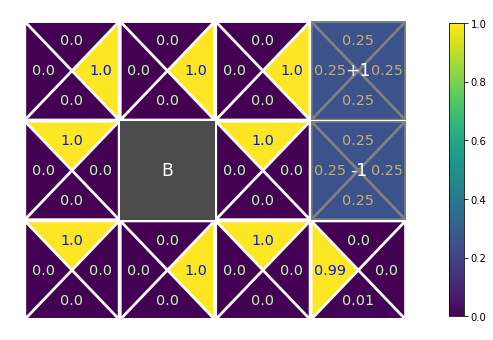}
  \caption{\textbf{Left}: The grid world MDP. The reward entering each state is shown on the grid. There is a block at state $(1,1)$ and $\mathcal A=\{\text{up, down, left, right}\}$.  Every action succeeds to its intended effect with probability $0.8$ and fails to each perpendicular direction with probability $0.1$. Hitting the walls means stay and $\gamma=.95$.  An episode starts in $(2,0)$ (bottom left corner) and any action in $(0,3)$ or $(1,3)$ transitions into $(2,0)$.  \textbf{Right}: Optimal Policy out of optimization~\eqref{ERL} for $\epsilon=.01$.}
  \label{fig:fig_gw}
\end{figure}

\begin{figure*}[h!]
      \centering
    \textbf{(1)}\includegraphics[scale=0.2]{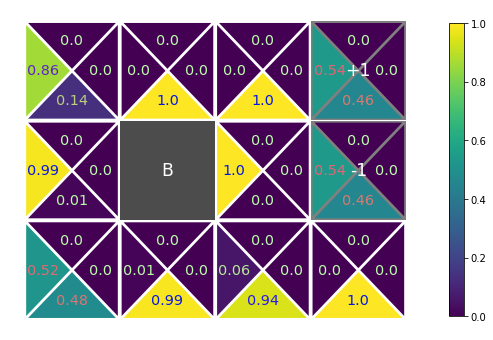}
    \textbf{(2)}\includegraphics[scale=0.2]{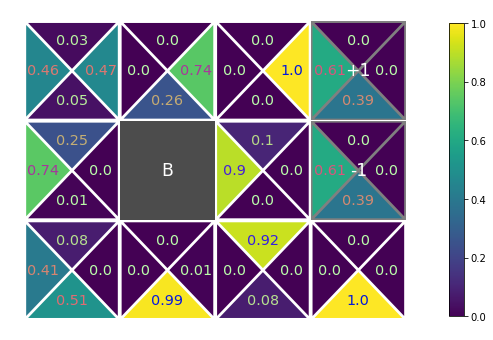}
    \textbf{(3)}\includegraphics[scale=0.2]{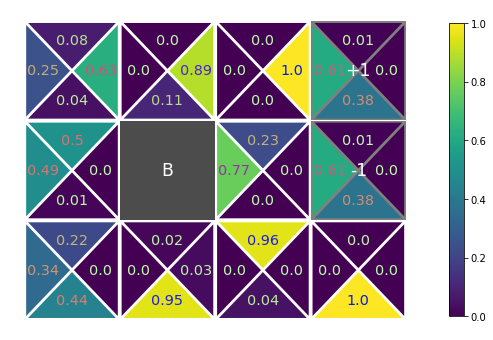}
    \textbf{(4)}\includegraphics[scale=0.2]{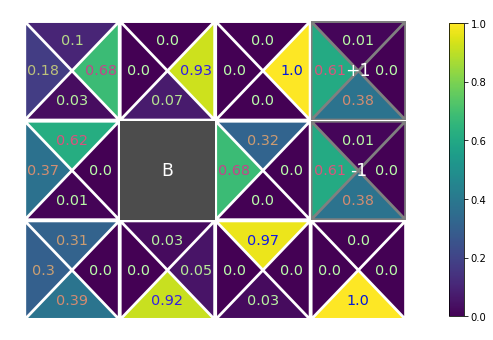}
    \textbf{(5)}\includegraphics[scale=0.2]{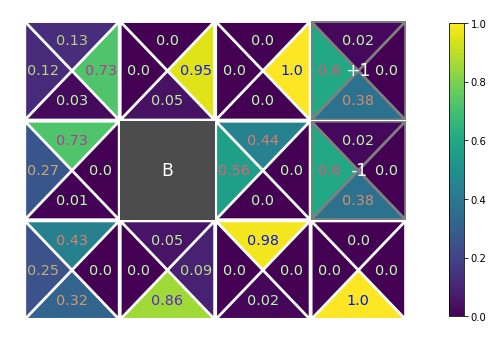}
    \textbf{(6)}\includegraphics[scale=0.2]{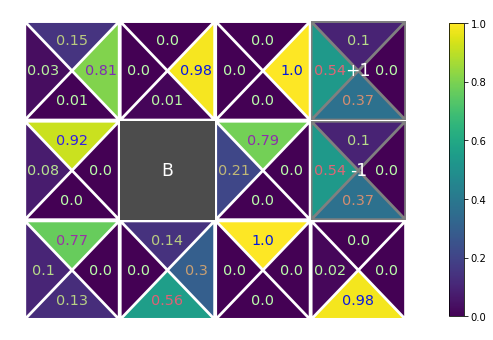}
    \textbf{(7)}\includegraphics[scale=0.2]{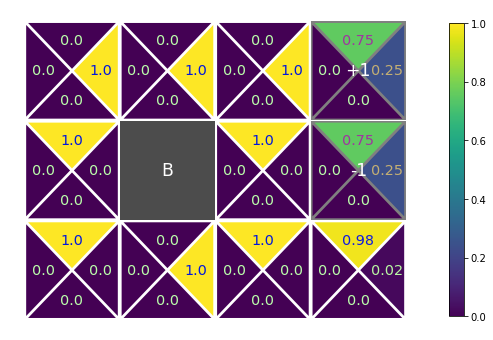}
    \textbf{(8)}\includegraphics[scale=0.18]{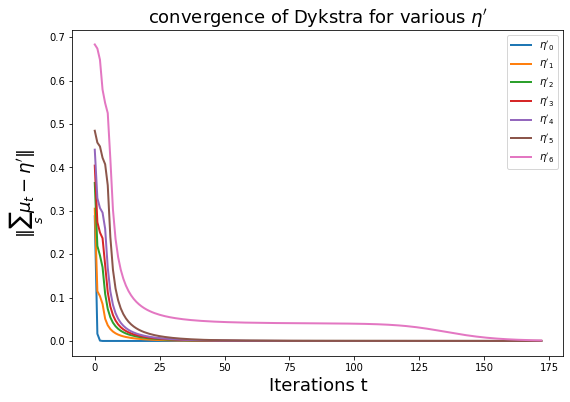}
    \caption{\textbf{(1-7)}: Optimal policy out of Dykstra, with  $\eta_i'=[\alpha,.5-\alpha,.5-\alpha,\alpha]$ for $\alpha_i \in\{e^{-10}, .1,.2,.25,.3,.4,.5-e^{-10}\}$ for $i\in \{1, \dots, 7\}$. 
    \textbf{(8)}:~corresponding convergence of the Dykstra in (1-7), horizontal axis is the number of iterations $t$ and the vertical axis is $\|\sum_s \mu_{t}-\eta'_i\|$.} \label{eta0550}
\end{figure*}

We first observe the independent effect of $\eta'$ on the policy, by setting $\epsilon_1=0$. We use, $\delta_{\eta'}(\mu^T\bm1)$ instead of $D_{\psi_1}(\mu^T\bm{1}\mid \eta')$ as an extreme case when $\epsilon_2\rightarrow \infty$ to focus on the role $\eta'$.
Figure~\ref{eta0550}(1-7), shows differences among policies when the marginal on actions shifts from a distribution where only 
equiprobable actions \textit{down} and \textit{left}  are allowed ($\eta'_1 = (0, 0.5,0.5,0)$)
towards the case where only \textit{up} and \textit{right} are permitted with equal probability ($\eta'_7 = ( 0.5,0,0,0.5)$).
~\footnote{The optimal policies in this section aren't necessarily deterministic (even though $\epsilon$ is set to be very small), because of the constraint $\delta_{\eta'}(\mu\bm(a))$. In general, the policies out of \eqref{DCRL2} are not necessarily deterministic either because of the nonlinear objective. }

In Figure~\ref{eta0550}(1), under $\eta'_1$, \textit{down} is the optimal action in state $(0,2)$ because, this is the only way to get $+1$ reward (with luck). In \ref{eta0550}(2), which changes to a $.1$ probability on \textit{right}, the policy eliminates the reliance on change by switching state $(0,2)$ to right.

Note that the optimal policy in Figure~\ref{fig:fig_gw}(left) does not include a \textit{down} move. 
When \textit{down} is forced to have non-zero probably, Figures~\ref{eta0550}(1-6), the policy assigns it to state $(2,3)$, 
towards the case where only \textit{up} and \textit{right} are permitted with equal probability ($\eta'_7 = ( 0.5,0,0,0.5)$).

Figures~\ref{eta0550}(7) shows the case where only \textit{up} and \textit{right} are allowed. 
In state $(2,3)$, this creates a quandary. \textit{Right} is unlikely to incur the $-1$ penalty, but will 
never allow escape from this state. For this reason, the policy selects \textit{up}, which with high
probability will incur to the $-1$ penalty, but has some probability of escape toward the $+1$ reward.

Figure \ref{eta0550}(8), depicts the convergence of Dykstra towards various $\eta'$. Notably, in all cases
the algorithm converges, and the rate of convergence increases following the order of the subfigures. 

Next, we test the extreme effect of constraints on the state marginals on the policy via various $\rho'$, by setting $\epsilon_2=0$ and $\epsilon_1$ very high. 
We study the policy when $\rho'(s)=.9$ for a single state $s$,  and uniform distribution of $.1$ on the rest of states other than $s$. Figure~\ref{fig:rho1}(1-3) 
shows the policies when $s\in \{(0,2),(1,2),(2,3)\}$. Hitting the wall seems to be viable strategy to stay and increase the visitation frequency of each of these states. 
Figure \ref{fig:rho1}(4) depicts the the convergence of Dykstra's algorithm towards various $\rho'$. As shown, the error never gets to zero. This is because by setting  $\epsilon_1\rightarrow \infty$, the objective is just to find an occupancy measure with closest state marginal to $\rho'$ and
$D_{\psi_1}(\mu\bm{1}\mid\rho')$ can never be zero if $\rho'$ is not from a valid occupancy measure.
\begin{figure}[h!]
\centering
 \begin{subfigure}
    \centering
    \textbf{(1)}\includegraphics[scale=0.2]{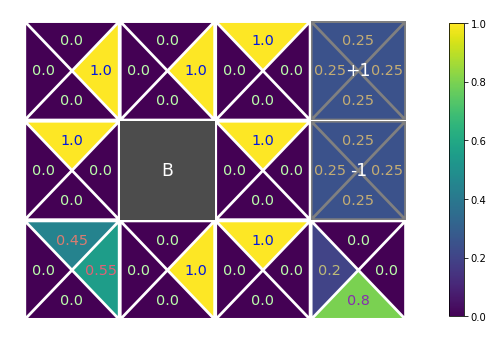}
    \textbf{(2)}\includegraphics[scale=0.2]{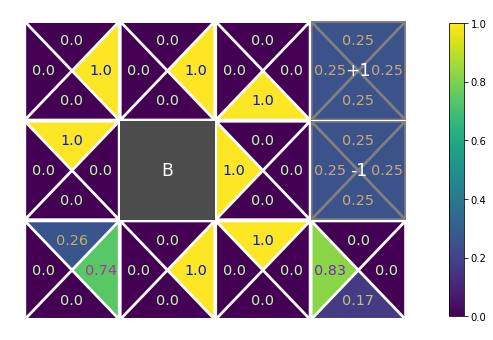}
    \textbf{(3)}\includegraphics[scale=0.2]{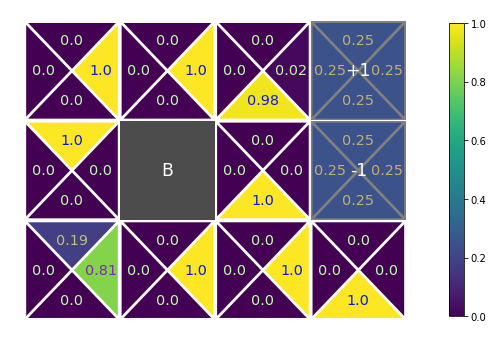}
    \textbf{(4)}\includegraphics[scale=0.18]{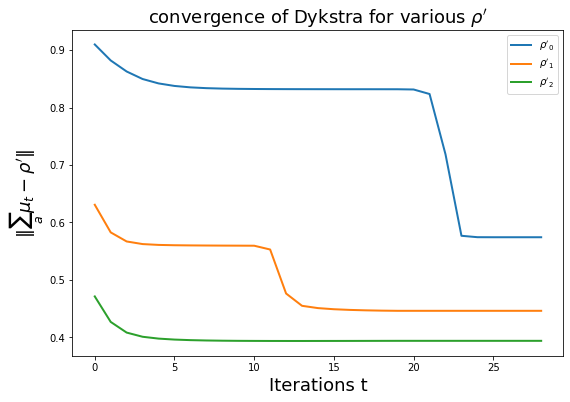}
    \caption{\textbf{(1-3)}: Optimal policy out Dykstra, when $\rho_i'\left(s\right)=.9$  for $s\in\{(0,2),(1,2),(2,3)\}$ and uniform distribution of $.1$ on the rest of states other than $s_i$ for $i\in\{1,2,3\}$. \textbf{(4)}:~corresponding convergence of the Dykstra in (1-3), horizontal axis is the number of iterations $t$ and the vertical axis is $\|\sum_a \mu_{t}-\rho'_i\|$.} 
    \label{fig:rho1}
    \end{subfigure}
\end{figure}

 We also test how imitating a policy with distributional constraints affects the learned policy.
 For this purpose, we create a new environment with reward $-10$ at state $(1,3)$. The optimal risk-averse policy $\pi_{1}$ 
 out of this environment is shown in Figure~\ref{fig:risk-averse}(left). 
 Let $\eta'_{1} = \sum_{s} \mu^{\pi_1}$ be the action marginal corresponding to $\pi_{1}$.
 Now consider the RL problem with reward of $-1$ in state $(1,3)$ constrained by $\eta'_{1}$.
 Figure~\ref{fig:risk-averse}(right) shows the resulting policy $\pi_{2}$.
 Notice that $\mu^{\pi_{2}}$ achieves the action marginal distribution $\eta'_{1}$, 
 however, $\pi_{2}$ is quite different from $\pi_{1}$, 
 since the unconstrained optimal policy for the environment with reward of $-1$ at state $(1,3)$ is more risk neutral. 
 In contrast, distributionally constraining $\rho^{\pi_1}=\sum_a \mu^{\pi_1}$ (as in previous experiments) results 
 in the same policy of $\pi_1$ as in Figure~\ref{fig:risk-averse}(left). 
 The differences are mostly in states $(0,3)$ and $(1,3)$, where actions can be freely chosen (but not their visitation frequency)
 and contribution of state $(2,3)$ which has a lower visitation probability.
 
\begin{figure}[h!]
\centering
 \begin{subfigure}
    \centering
    \includegraphics[scale=0.25]{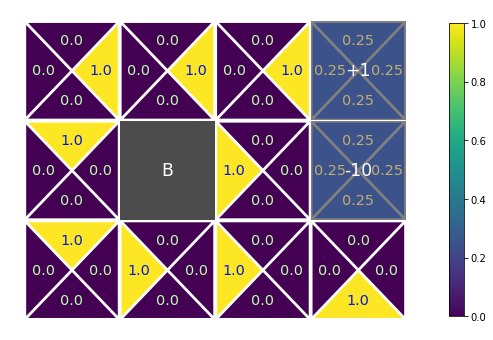}
    \hspace{0.4in}
    \includegraphics[scale=0.25]{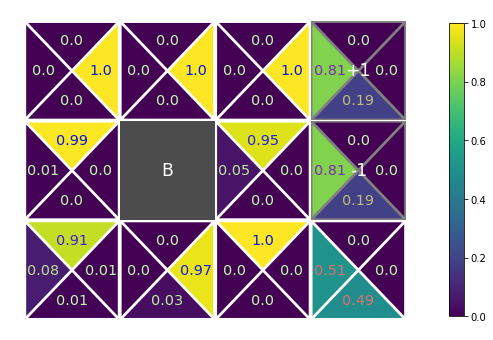}
    \caption{\textbf{Left}: Risk averse optimal policy learned by changing $-1$ to $-10$ in state $(1,3)$. 
    \textbf{Right}: optimal policy using $\eta'$ out of risk-averse policy as a distributional constraint.}
    \label{fig:risk-averse}
    \end{subfigure}
\end{figure}
Consider constraints on both $\eta'$ and $\rho'$. As explained earlier, the policy will then get as close as possible to $\rho'$ while satisfying the action distribution $\eta'$. Figure \ref{fig:no_up}(left) shows the optimal policy for $\eta'$  with no constraint on $\rho'$. $\eta'$  is a distribution where no \textit{up} is allowed and the other three actions equiprobable. Figure~\ref{fig:no_up}(right) depicts the policy under constraints on both $\eta'$ and $\rho'$ when $\rho'$ is the same distribution in Figure \ref{fig:rho1}(1). The leftmost column and top row of this policy leads to (0,2) but in an attempt to satisfy $\rho'$, the policy goes back to the \textit{left}. 
\begin{figure}[h!]
\centering
 \begin{subfigure}
    \centering
    \includegraphics[scale=0.25]{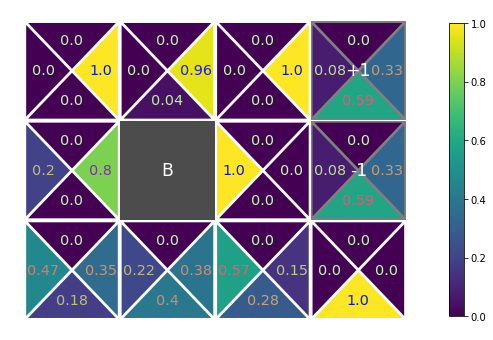}
    \hspace{0.4in}
    \includegraphics[scale=0.25]{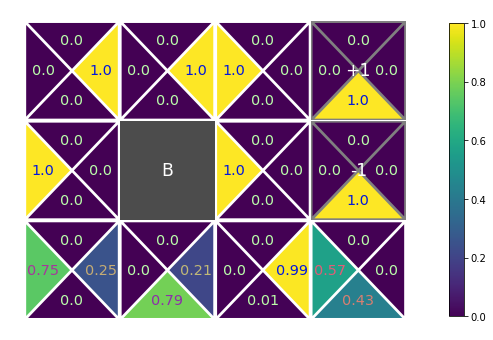}
    \caption{\textbf{Left}: Optimal policy with $\eta'=[e^{-10},\alpha,\alpha,\alpha]$, $\alpha=(1-e^{-10})/3$. \textbf{Right}: Policy with the same $\eta'$ and $\rho_1'$.} \label{fig:no_up}
    \end{subfigure} 
\end{figure}

\section{Related Works}
\label{sec:discussion}
In this section we review and discuss the related works to our proposed approach.
    \paragraph{$\bullet$ Objectives and Constraints in Reinforcement Learning.} 
    Posing policy optimization as a constrained linear programming on the space of occupancy measures has been long studied \cite{puterman2014markov}. Recent works have expanded linear programming view through a more general convex optimization framework. For example, \cite{neu2017unified} unifies policy optimization algorithms in literature for entropy regularized average reward objectives. \cite{nachum2019algaedice,nachum2019dualdice,nachum2020reinforcement} propose entropy-regularized policy optimizations under Fenchel duality to incorporate the occupancy measure constraint. Unlike these works, we  looked at policy optimization from an optimal transport point of view.   To do so, we proposed a structured general objective based on Bregman divergence that allows considering relaxations of entropy regularization using marginals.
    \cite{zhang2020variational} studies a general concave objective in RL and proposes a variational Monte Carlo algorithm using the Fenchel dual of the objective function. Similar to these works we take advantage of Fenchel duality. However, other than different view point  and  structured objective, our work differs in solving the optimization by breaking the objective using Dykstra's algorithm.
    
    \cite{zhang2020cautious} proposes various \textit{caution penalty functions} as the RL objective and a primal-dual approach to solve the optimization. One of these objectives is a penalty function on $\K(\cdot|\rho')$, which is a part of our proposed unbalanced formulation. Other than our problem formulation, in this work, we focused on distributional penalty on global action execution  which, to the best our knowledge, has not been studied before. 
    
    In constrained MDPs, a second reward $c(s,a) < 0$ is used to define an constrained value function $C^{\pi}$ \cite{altman1999constrained, geibel2006reinforcement}.
    Here $C^{\pi}(s) = \mathbb{E}_{\pi}\left[ \sum_{t=0}^{\infty} \gamma^{t} c(s, a)\right]$ and the constraint is in form of $\mathbb{E}_{\pi} \left[C^{\pi}(s)\right] > A$ $(*)$, 
    where $A$ is a constant. Thus considering $c(s,a)$ as the cost for taking action $a$ at state $s$, 
    constrained MDP optimizes the policy with a fixed upper bound for the expected cost. 
    Rather than introducing a fixed scalar restriction (A), 
    our formulation allows distributional constraints over both the action and state spaces (i.e. $\rho'$ and $\eta'$). 
    The source of these distributional constraints may vary from an expert policy to the environmental costs and we can apply them via penalty functions.
    In special cases, when an action can be identified by its individual cost, constraint $(*)$ on expected cost can be viewed as a special case of marginal constraint on $\eta'$.
    For instance, in the grid world of Figure \ref{fig:fig_gw}, if the cost for \textit{up} and  \textit{right} is significantly higher than \textit{down} and \textit{left},
    then limited budget (small expected cost) is essentially equivalent to having a marginal constraint $\eta'_1$ supported on \textit{down} and \textit{left}. 
    However, in general, when cost $c(s, a)$ varies for state per action, 
    the expected cost does not provide much guidance over global usage of actions or visitation of states as our formulation does.

   


    \paragraph{$\bullet$ Reinforcement Learning and Optimal Transport}
        Optimal transport in terms of Wasserstein distance has been proposed in the RL literarture. \cite{zhang2018policy} views policy optimization as Wasserstein gradient flow on the space of policies. \cite{pacchiano2020learning} defines behavioral embedding maps on the space of trajectories and uses an approximation of Wasserstein distance between measures on embedding space as regularization for policy optimization. Marginals of occupancy measures can be viewed as embeddings via state/action distribution extracting maps.  Our work defines an additive structure on these embedding functionals which is broken into Bregman projections using Dykstra.
    \paragraph{$\bullet$ Imitation Learning and Optimal Transport}
        In the imitation learning literature, \cite{xiao2019wasserstein} proposed an adversarial inverse reinforcement learning method which minimizes the Wasserstein distance to the occupancy measure of the expert policy using the dual formulation of optimal transport. \cite{dadashi2020primal} minimized the primal problem of Wasserstein minimization and \cite{papagiannis2020imitation} minimize the \textit{Sinkhorn Divergence} to the expert's occupancy measure. These works are fundamentally different from our approach as we are not solving the inverse RL problem and we view RL itself as a problem of stochastic assignment of actions to states. The type of distributional constraints via unbalanced optimal transport proposed in our work can be considered as relaxation of the idea of matching expert policy occupancy measures. We consider matching the distribution of global action executions and state visitations of the expert policy.
    \paragraph{$\bullet$ Related works in Optimal Transport} It is not the first time (unbalanced) optimal transport is considered constrained. 
        Martingale optimal transport imposes an extra constraint on the mean of the coupling \cite{beiglbock2013model} and using entropic regularization, 
        Dykstra can be applied \cite{henry2013automated}.
   \paragraph{$\bullet$ Other Settings} 
   In unbalanced OT formulation  of \eqref{DCRL} we used penalty functions $D_\psi$. One can apply other functions like the indicator function to enforce constraints on marginals like $\delta_{\eta'}(\mu^T\bm{1})$ as discussed in section \ref{sec:DCPO}. However, using the constraint $\delta_{\rho'}(\mu\bm{1})$ in \eqref{DCRL} could be problematic as it can easily be incompatible with the occupancy measure constraint. If  $\rho'$ is not coming form a policy, then the optimization is infeasible. Despite this, setting $\epsilon_2=0$ and taking $\rho' =\rho^{\pi_{k-1}}$ in \eqref{DCRL} for the $k$'th iteration of an iterative policy optimization algorithm, \eqref{DCRL} results in objective similar to TRPO \cite{schulman2015trust}. 
    
\section{Conclusion}

We have introduced distributionally-constrained policy optimization via unbalanced optimal transport. Extending prior work, we recast RL
as a problem of unbalanced optimal transport via minimization of an objective with a Bregman divergence which is optimized through Dykstra's algorithm. 
We illustrate the theoretical approach through the convergence and policies resulting from marginal constraints on $\eta'$ and $\rho'$ both individually and together. The result unifies different perspectives on RL that naturally allows incorporation of a wide array of realistic constraints on policies.

%% file: SI.tex
\section{Fenchel Dual and Proofs in Section 3}\label{apd:sec3}
For any function $f:\text{Dom}(f)\rightarrow \mathbb R$, its convex conjugate (or Fenchel dual) $f^*$ is defined as $f^*(y)=\max_{x}\langle x,y\rangle -f(x)$. 
If $f$ is proper, convex and lower-semi continuous then $f^*$ has the same properties and one can write $f(x)=\max_{y} \langle x,y\rangle-f^*(y)$. If $f$ is strictly convex and smooth on $\text{int}(\text{dom}f)$, then $\nabla f$ and $\nabla f^*$ are bijective maps between $\text{int}(\text{dom}f^*)$ 
and \text{int}(\text{dom}f), i.e., $\nabla f^* =\nabla f^{-1}$. It is easy to verify that
\begin{itemize}
    \item  For $f(x)=\delta_{a}(x)$,  $f^*(y)=\langle a,y\rangle$ .
    \item 
    Consider $f$ with $\text{Dom}(f) = \mathcal{M}(Z)$, where $Z$ is an underlying space. For a fixed $p \in \mathcal{M}(Z)$,
    if $f(x)=D_\psi(x|p)$, $f^*(y)=\mathbb E_{p}(\psi^*(y))$.
    Also, if  $\psi(z) =z\log z$, then $f(x)=\text{KL}(x| p)$ and $f^*(y)=\log \mathbb E_{p}(\exp(y))$.
\end{itemize}
\begin{proof}[Proof of Proposition \ref{prop:Dykestra}]
Under condition \eqref{relint_condition}, the Fenchel-Legendre duality holds and the solution of optimization \eqref{General_OT} can be recovered via
\begin{equation}
\label{dual_general_OT}
\begin{split}
\underset{u_1,\cdots,u_N}{\max} -\sum_{i=1}^N \phi_i^*(u_i)-D_{\Gamma}^*\left(-\sum_{i=1}^N u_i|\xi\right) =
\underset{u_1,\cdots,u_N}{\max} -\sum_{i=1}^N \phi_i^*(u_i)-\Gamma^*\left(\nabla\Gamma(\xi)- \sum_{i=1}^N u_i\right) &-\langle\nabla\Gamma(\xi),\xi\rangle+\Gamma(\xi)
\end{split}
\end{equation}
with the primal-dual relationship
\[
\mu =\nabla\Gamma^*\left(-\sum_{i=1}^N u_i\right).
\]
Applying coordinate descent on  \eqref{dual_general_OT}, with initial condition $(u_1^{(0)},\cdots,u_N^{(0)})=(0,\cdots,0)$, and setting $i=[l]_N$, $J =\{1,\cdots,N\}\setminus\{i\}$, at $l>0$ we get the iteration 
\begin{equation}
\label{coordinate descent}
\begin{aligned}
&u_i^{(l)} =\underset{u_i}{\arg\max} \, -\phi_i^*(u_i)-\Gamma^*(q-u_i),\\
&u_j^{(l)} = u_j^{(l-1)}, \forall j\in J,\\ 
\end{aligned}
\end{equation}
where $q=\nabla\Gamma(\xi)-\sum_{j=1}^Ju_j^{(l-1)}$.
The primal problem of optimization in \eqref{coordinate descent} is 
\begin{equation*}
\begin{split}
    \underset{\mu_i}{\arg\min}\, \Gamma(\mu_i)-\langle q,\mu_i\rangle+\phi_i(\pi_i) = \underset{\mu_i}{\arg\min}\, D_\Gamma\left(\mu_i|\nabla\Gamma^*(q)\right)+\phi_i(\mu_i)+\text{const} =\text{Prox}_{\phi_i}^{D_\Gamma}(\nabla\Gamma^*(q)).
\end{split}
\end{equation*}
Hence, under the relation $\mu_i =\nabla\Gamma^*(q-u_i)$, we can rewrite \eqref{coordinate descent} as
\begin{equation*}
u_i^{(l)} =q-\nabla\Gamma \left(\text{Prox}_{\phi_i}^{D_\Gamma}(\nabla\Gamma^*(q))\right).
\end{equation*}
Hence, for $i=[l]_N$, we have $\mu_i^{(l)}=\nabla\Gamma^*(q-u_i^{(l)})$
\begin{equation*}
\begin{split}
\mu^{(l)} &=\nabla\Gamma^*\circ\nabla\Gamma \left(\text{Prox}_{\phi_i}^{D_\Gamma}(\nabla\Gamma^*(q))\right) =\text{Prox}_{\phi_i}^{D_\Gamma}(\nabla\Gamma^*(q))= \text{Prox}_{\phi_i}^{D_\Gamma}\left(\nabla\Gamma^*\left(q- u_i^{(l-N)}+ u_i^{(l-N)}\right)\right)\\ &=\text{Prox}_{\phi_i}^{D_\Gamma}\left(\nabla\Gamma^*\left(\nabla \Gamma(\mu^{(l-1)})+u_i^{(l-N)}\right)\right).
\end{split}
\end{equation*}
Calculating the difference $ \nabla\Gamma(\mu^{(l-1)}) -\nabla\Gamma(\mu^{(l)})$ and change of variable $\nu^{(l)} = -u_{[l]_N}^{(l)}$, ends the proof.
\end{proof}
\begin{proof}[Proof of Corollary \ref{cor:dykstra_kl}]\cite{peyre2015entropic}
Setting $\Gamma=\mathcal H$, then $\nabla \Gamma=\log$ and $\nabla \Gamma^*=\exp$. Also,   $D_{\Gamma}=\K$ and by change of variable $z^{(l)}=\nabla \Gamma(v^{(l)})$, Corollary follows.
\end{proof}

\section{Derivations in Section 4}\label{apd:sec4}
Here we derive the Proximal operators  in section~\eqref{sec:DCPO}.
\subsection{Proximal Operator Calculations when $D_\Gamma=\mathrm{KL}$}
When $\normalfont D_\Gamma=\kl$ (i.e., $\Gamma =\mathcal H$), for $\phi_3(\mu)=\delta_{b^\mu}(A^\mu \mu)$ we have:
\begin{equation}
\label{phi_3_app_1}
\text{Prox}_{\phi_3}^{\text{KL}} (\mu)=\arg \underset{\tilde \mu \in \Delta}\min\, \text{KL}(\tilde \mu|\mu),
\end{equation}
letting  $V,\lambda$ to be dual variables, by the definition of $\Delta$ in lemma \ref{syed}, the lagrangian  of \eqref{phi_3_app_1} is
\[
\begin{split}
\K(\tilde \mu|\mu)-\gamma\sum_{s,s',a'}P(s|s',a')\mu(s',a')V(s)-(1-\gamma)\sum_s p_0(s)V(s)+\sum_{s,a} \mu(s,a)V(s)+\lambda \left( \sum_{s,a}\mu(s,a)-1\right).
\end{split}
\]
The derivative of the Lagrangian with respect to $\tilde \mu(s,a)$ for \eqref{phi_3_app_1} is
\[
\log(\tilde \mu(s,a)/\mu(s,a))-\gamma\sum_{s'}P(s'|s,a)V(s')+V(s)+\lambda=0.
\]
So, the optimal solution is
\[
\tilde \mu(s,a) =e^{-\lambda}\mu(s,a)\exp(\gamma \sum_{s'}P(s'|s,a)V(s')-V(s)).
\]
Since $\sum_{s,a}\tilde\mu(s,a)=1$, we have $$\lambda=\log\sum_{s,a}\mu(s,a)\exp \left (\gamma \sum_{s'}P(s'|s,a)V(s')-V(s)\right),$$ then
\begin{equation}
\label{primal}
\tilde \mu(s,a) =\frac{\mu(s,a)\exp \left (\gamma \sum_{s'}P(s'|s,a)V(s')-V(s)\right)}{\sum_{s,a}\mu(s,a)\exp(\gamma \sum_{s'}P(s'|s,a)V(s')-V(s))},
\end{equation}
where $V$ is the solution of the dual problem
\begin{equation}
\label{dual}
\begin{split}
&\min_{\lambda,V} \lambda +(1-\gamma)\sum_s p_0(s)V(s)= \min_{V} \log\sum_{s,a}\mu(s,a)\exp(\gamma \sum_{s'}P(s'|s,a)V(s')-V(s))+(1-\gamma)\sum_s p_0(s)V(s).
\end{split}
\end{equation}
By solving  optimization \eqref{dual}, we can recover optimal $\tilde \mu$ by equation \eqref{primal}.

For other proximal operators we need the following lemma:
\begin{lemma}\cite{peyre2015entropic}\label{main_app} 
For any convex function $h$:
\begin{itemize}
    \item[(i)] For any $\phi(\mu)=h(\mu\bm{1})$
    \[
    \text{Prox}_{\phi}^{\text{KL}}(\mu) = \text{diag}\left(\frac{\text{Prox}_h^{\text{KL}}(\mu \bm{1})}{\mu \bm{1}}\right)\mu,
    \]
    \item[(ii)] and, if $\phi(\mu)=h(\mu^T\bm{1})$
    \[
    \text{Prox}_{\phi}^{\text{KL}} (\mu)=\mu\, \text{diag}\left( \frac{\text{Prox}_{h}^{\text{KL}}(\mu^T\bm{1})}{\mu^T\bm{1}}\right).\label{prox2}\\
    \]
\end{itemize}
\end{lemma}
\begin{proof}
Let
\[
\tilde \mu :=\text{Prox}_{\phi}^{\text{KL}}(\mu) = \arg\min_{\tilde \mu} \text{KL}(\tilde \mu | \mu) + h(\tilde \mu\bm{1}),
\]
hence, by the first order condition, at the optimal $\tilde \mu$, there exists $Z \in \partial h(\tilde \mu \bm{1})$ such that we have $\log{\frac{\tilde \mu}{\mu}}+Z=0$. Therefore,
\begin{equation}
\label{PROX_Z}
    \tilde \mu = \text{diag}\left(e^{-Z}\right)\mu.
\end{equation}
Similarly, considering the first order condition corresponding to  $\tilde u:=\text{Prox}_h^{\text{KL}}(\mu\bm{1})$, we get $\tilde u = \text{diag}(e^{-Z})\mu \bm{1}$, and combining this with \eqref{PROX_Z}  proves part (i). Transposing (i), results in (ii).
\end{proof}

Using this lemma,  we calculate some of the proximal operators we used in Section \ref{sec:DCPO} and \ref{sec:demo}.
\begin{itemize}
    \item 
Let $\rho :=\mu\bm{1}$. Given $\rho' \in \mathcal M(\mathcal S)$, when $\phi_1(\mu)=\epsilon_1 D_{\psi_1}(\rho|\rho')=\epsilon_1 \K(\rho|\rho')$ , 
$$
\text{Prox}_{\epsilon_1\K}^{\K}(\rho) =\arg\min_{\tilde \rho} \K(\tilde \rho|\rho)+\epsilon_1 \K(\tilde \rho|\rho' ),
$$
then, by the first order condition, we have
$$
\log\frac{\tilde \rho}{\rho}+\epsilon_1 \log\frac{\tilde \rho}{\rho'}=0,
$$
and $\text{Prox}_{\epsilon_1\K}^{\K}(\mu\bm{1}) =(\rho {\rho'}^{\epsilon_1})^{1/(1+\epsilon_1)}$. Applying lemma \ref{main_app}(i), gives the proximal operator for $\phi_1(\mu)$.
\item When $\phi_1(\mu)=\delta_{\rho'}(\rho)$, then $\text{Prox}^{\text{KL}}_{\phi_1}(\rho)=\rho'$ and by applying lemma  \ref{main_app}(i), we get 
 \[
    \text{Prox}_{\phi_1}^{\text{KL}}(\mu) = \text{diag}\left(\frac{\rho'}{\mu \bm{1}}\right)\mu.
    \]
\item
Similarly,  given $\eta' \in \mathcal M(\mathcal A)$. If $\eta  := \mu^T \bm{1}$ and $\phi_2=\epsilon_2 D_{\psi_2}(\eta|\eta')=\epsilon_2 \K(\eta|\eta')$,
$$
\text{Prox}_{\epsilon_2\K}^{\K}(\eta) =\arg\min_{\tilde \eta} \K(\tilde \eta|\eta)+\epsilon_2 \K(\tilde \eta|\eta' ),
$$
and $\text{Prox}_{\epsilon_2\K}^{\K}(\mu^T\bm{1}) =(\eta {\eta'}^{\epsilon_2})^{1/(1+\epsilon_2)}$. Applying lemma \ref{main_app}(ii) results in the proximal operator for $\phi_2$. \item Also, when $\phi_2(\mu)=\delta_{\eta'}(\mu^T\bm{1})$,  by lemma \ref{main_app}(ii), we get 
    \[
    \text{Prox}_{\phi_2}^{\text{KL}} (\mu)=\mu\, \text{diag}\left( \frac{\eta'}{\mu^T\bm{1}}\right).
    \]
\end{itemize}
\subsection{Convergence of the iterative policy optimization for problems \eqref{DCRL} and \eqref{DCRL2}}
\label{sec:conv_big_loop}


We now show the convergence and policy improvement of optimization \eqref{DCRL}, and \eqref{DCRL2} in an iterative policy optimization scenario.

\begin{proposition} 
  Let $(\mu_k)_{k\in\mathbb{N}}$ be a sequence of occupancy measures such that
  for each $k\ge1$, $\mu_k$ is the
  solution to the optimization  \eqref{DCRL} with $\mu'=\mu_{k-1}$,
  then $\lim\limits_{k\rightarrow\infty}\mu_k = \mu^{\ast}$ is the solution to 
  $\underset{\mu \in \Delta}\max \mathbb E_\mu[r]$.
 \end{proposition}
\begin{proof}
  Let $\mu_0 \in \Delta$ such that $\mu_0(s,a)>0$ for all $(s,a)$.
Let $\mu_k=\Psi(\mu_{k-1})$ for all $k>0$ where $\Psi(\mu')$ is the solution
  to \eqref{DCRL2}.
  Let $\Theta(\mu)= \mathbb E_\mu[r]$, since
  $\mu_{k-1}$ is a feasible point for optimization \eqref{DCRL}, we have monotonic improvement on the discounted accumulated rewards as
  \begin{align}
      \Theta(\mu_{k-1}) \leq 
      \max_{\mu \in \Delta} \Theta(\mu)- \epsilon_1 D_{\psi_1}(\mu\bm{1} | \mu_{k-1}\bm{1}) - \epsilon_2 D_{\psi_2}(\mu^T\bm{1} | \mu_{k-1}^T\bm{1})  \leq \Theta(\mu_{k}),
  \end{align}
  with equality achieved only if $\mu_{k-1}=\mu_{k}=\mu^{\ast}$.

  $\Theta(\mu)-\epsilon_1 D_{\psi_1}(\sum_{a}\mu|\sum_{a}\mu') - \epsilon_2 D_{\psi_2}(\sum_{s}\mu|\sum_{s}\mu')$ is
  strictly concave on $\mu$ and smooth on $\mu'$. 
  Thus $\Psi(\mu')$ is continuous
  and so is $\Theta(\Psi(\mu'))$. 

  Since we assumed $\mathcal S, \mathcal A$ are finite, from monotone convergence theorem, we can conclude that
  $\lim_{k\rightarrow\infty}\Theta(\mu_{k})=:\theta_\infty$ exists.
  Then we prove the theorem by contradiction. Suppose that
  $\theta_\infty<\Theta(\mu^\ast)=\mathrm{max}_{\Delta_\eta}\Theta$,
  then $\Theta^{-1}(\theta_\infty)$ is closed and bounded, thus
  compact in $\Delta$, therefore there exists a $c>0$ such that
  $\Theta(\Psi(\mu))-\Theta(\mu)\ge c$ on $\Theta^{-1}(\theta_\infty)$.
  By continuity and compactness we may choose a $\delta<c/2$ small
  enough such that $\Theta(\Psi(\mu))-\Theta(\mu)>c/2$ on
  $\Theta^{-1}([\theta_\infty-\delta,\theta_\infty])$.
  There is such a $\delta(\mu)>0$ for
  every $\mu\in\Theta^{-1}(\theta_\infty)$ and we can choose $\delta=\min\delta(\mu)>0$ by compactness.

  Since $\lim_{k\rightarrow\infty}\Theta(\mu_k)=\theta_\infty$, there exists
  an $n>0$ such that $\Theta(\mu_n)\in[\theta_\infty-\delta,\theta_\infty]$,
  thus
  $\Theta(\mu_{n+1})=\Theta(\Psi(\mu_n))>\Theta(\mu_n)+c/2\ge\theta_\infty-\delta+c/2>\theta_\infty$.
  This is an contradiction to the assumption that $\Theta(\mu_k)$ converges to
  $\theta_\infty$ increasingly.

  Therefore, $\theta_\infty=\Theta(\mu^\ast)$, and from
  strict convexity of $\Theta$, we have
  $\Theta^{-1}(\theta_\infty)=\{\mu^\ast\}$. So
  $\lim_{k\rightarrow\infty}\mu_k=\mu^\ast$.
\end{proof}

Similarly, for the optimization \eqref{DCRL2}:
\begin{align*}
 \underset{\mu \in \Delta }\max -\K(\mu \mid \xi) - &\epsilon_1 D_{\psi_1}\left(\mu\bm{1}\mid  \rho'\right) 
-\epsilon_2 D_{\psi_2}\left(\mu^T\bm{1}\mid \eta'\right),
\end{align*}
we can have the following proposition:
\begin{proposition}
  Let $(\mu_k)_{k\in\mathbb{N}}$ be a sequence of occupancy measures such that
  for each $k\ge1$, $\mu_k$ is the
  solution to the optimization problem \eqref{DCRL2} with $\mu'=\mu_{k-1}$,
  then $\lim\limits_{k\rightarrow\infty}\mu_k = \mu^{\ast}$ is the solution to 
  $\underset{\mu \in \Delta}\max -\K(\mu \mid \xi)$.
\end{proposition} 
 
\section{Section \ref{AC} Derivations}\label{apd:sec4_1}
Here we derive optimization \eqref{ac1}. Let's fix policy $\pi$, then following the policy optimization approach in \cite{nachum2020reinforcement},  one might define the policy evaluation problem for $\pi$ as $\max_\mu h(\mu)-\delta_{b^\pi}(A^\pi \mu)$ for any arbitrary concave function $h$ as the problem is over-constrained and $\mu$ is the unique solution of $A^\pi \mu =b^\pi$. So we define the policy evaluation problem for a fixed $\pi$ corresponding to \eqref{DCRL} as
\begin{equation}
\label{appendix_DCRL0}
\begin{split}
&\max_\mu  \mathbb E_{\mu}\left[r\right]- \epsilon_1 D_{\psi_1}(\mu\bm{1}\mid \rho') - \epsilon_2 D_{\psi_2}(\mu^T\bm{1}\mid \eta')-\delta_{b^\pi}(A^\pi \mu)\\
&\stackrel{(a)}{=} \max_\mu \min_{u,v,Q} \mathbb \langle\mu,r\rangle -\epsilon_1\langle \mu\bm{1},u\rangle +\epsilon_1\mathbb E_{\rho'}[\psi_1^*({u(s))}]-\epsilon_2 \langle \mu^T\bm{1},v\rangle+\epsilon_2 \mathbb E_{\eta'}[\psi_2^*(v(a))]-\langle A^\pi\mu,Q\rangle
+\mathbb E_{b^\pi}[Q(s,a)] 
\\
&=\max_\mu \min_{u,v,Q} \mathbb \langle \mu, r-A^{{\pi}*} Q -\epsilon_1 u\bm{1_A}^T-\epsilon_2 \bm{1_S}v^T\rangle 
+\mathbb E_{b^\pi}[Q(s,a)]+\epsilon_1\mathbb E_{\rho'}[\psi_1^*({u(s))}]
+\epsilon_2 \mathbb E_{\eta'}[\psi_2^*(v(a))],
\end{split}
\end{equation}
where (a) is obtained by replacing the last three terms by their convex conjugates and $A^{{\pi}*}$ is transpose of $A^{{\pi}}$. This gives optimization \eqref{ac1}.
If we regularize  objective \eqref{ac1}  with  $D_{\psi}(\mu\mid \mu')$ regularization, under Fenchel duality  we get
\begin{equation}
\label{appendix_DCRL1}
\begin{split}
&\max_\mu \min_{u,v,Q} \mathbb \langle \mu, r-A^{{\pi}*} Q -\epsilon_1 u\bm{1_A}^T-\epsilon_2 \bm{1_S}v^T\rangle-D_{\psi}(\mu\mid \mu')
+\mathbb E_{b^\pi}[Q(s,a)] +\epsilon_1\mathbb E_{\rho'}[\psi_1^*({u(s))}]
+\epsilon_2 \mathbb E_{\eta'}[\psi_2^*(v(a))]\\
=& \min_{u,v,Q} \left \{\!-\min_\mu\langle\!-\!\mu, \!r\!-\!A^{{\pi}*} Q\!-\!\epsilon_1\!u\bm{1_A}^T-\epsilon_2\!\bm{1_S}v^T\rangle\!+\!D_{\psi}(\mu\!\mid\! \mu')\!\right\}
+\mathbb E_{b^\pi}[Q(s,a)]\!+\!\epsilon_1\mathbb E_{\rho'}[\psi_1^*({u(s))}]
+\epsilon_2 \mathbb E_{\eta'}[\psi_2^*(v(a))]\\
=& \min_{u,v,Q} \mathbb E_{\mu'}\left[\psi^*\left( r(s,a)-A^{{\pi}*} Q(s,a)-\epsilon_1 u(s)-\epsilon_2 v(a)\right)\right]
+\mathbb E_{b^\pi}[Q(s,a)] +\epsilon_1\mathbb E_{\rho'}[\psi_1^*({u(s))}]
+\epsilon_2 \mathbb E_{\eta'}[\psi_2^*(v(a))],
\end{split}
\end{equation}
and wrapping \eqref{appendix_DCRL1} with $\max_{\pi}$ results in \eqref{ac2_pe}.
The gradient in equations \eqref{20}, \eqref{21} and \eqref{22} are basic calculus. Assuming $Q^*,u^*,v^*$ are optimal functions out of the policy evaluation in \eqref{appendix_DCRL1} for the given fixed $\pi$, then using Danskin's theorem, \cite{bertsekas1999nonlinear} we have $\partial_\pi \min_{u,v,Q} \mathcal L(u,v,Q;\pi)= \partial_\pi  \mathcal L(u^*,v^*,Q^*;\pi)$ and 
gradient in \eqref{23} is derived using the facts: \textbf{(A)} $P^\pi Q(s,a)=\mathbb E_{ {s'\sim P(\cdot|s,a)},{a'\sim\pi(\cdot|s')}}[Q(s',a')]$ and  \textbf{(B)} for any distribution $z \sim p$, $\partial_p\mathbb E_p[h(z)]=\mathbb E_p[h(z)\nabla \log p(z)]$. 

In order to derive the objective of \eqref{ac33}, first we interchange $\min_{\mu}$ and $\max_{u,v,Q}$ in \eqref{appendix_DCRL0} by minimax theorem. 
Fixing policy $\pi$, we rewrite the first term in \eqref{appendix_DCRL0} as
\[
\begin{split}
\langle \mu, r-A^{{\pi}*} Q -\epsilon_1 u\bm{1_A}^T-\epsilon_2\bm{1_S}v^T\rangle & = \mathbb E_\mu \left[r(s,a)-A^{{\pi}*} Q(s,a) -\epsilon_1 u(s)-\epsilon_2 v(a)\right] \\&= \mathbb E_{\mu'} \left[\zeta(s,a)\left(r(s,a)-A^{{\pi}*} Q(s,a) -\epsilon_1 u(s)-\epsilon_2 v(a)\right)\right],
\end{split}
\]
where $\zeta(s,a)=\frac{\mu(s,a)}{\mu'(s,a)}$.

Equations~\eqref{25}--\eqref{28} are basic calculus derivations. For \eqref{29},
given optimized $Q, u, v$ and $\zeta$, using \textbf{(A)} and \textbf{(B)}, we can take the gradient of the first two terms in \eqref{appendix_DCRL0}, i.e.,
\begin{equation*}
\begin{split}
\mathbb E_{\mu'} \left[\zeta(s,a)\left(r(s,a)-A^{{\pi}*} Q(s,a) -\epsilon_1 u(s)-\epsilon_2 v(a)\right)\right]+\mathbb E_{b^\pi}[Q(s,a)],
\end{split}
\end{equation*} 
with respect to $\pi$ and combine them under the relation $\mu(s,a)=(1-\gamma) p_0(s)\pi(a|s)+\gamma\pi(a|s)\sum_{s',a'} P(s|s',a') \mu(s',a')$  to get \eqref{29}. 

These two approaches in Section \ref{sec:DCPO}, are similar to policy gradient derivations in \cite{nachum2020reinforcement}, with corrective terms on $u(s)$ and $v(a)$ in the softmax operator defined in the main text.




 
 

\section{Demonstrations Setup}\label{apd:demo}
In Section \ref{sec:demo}, we set $D_{\psi_1}=D_{\psi_2}=\K$ and we applied Dykstra in the gridworld. We used the Frobenius norm on the difference of two consecutive matrices out of Dykstra until the error is less than $10^{-5}$. 


 In order to see the extreme effect of $\rho'$ independently (setting $\epsilon_2=0$), we can enforce it as $\delta_{\rho'}(\mu 1)$, even though  it might not be possible to find a $\mu$ such $\mu \bm{1}=\rho'$ as discussed in Section~\ref{sec:demo}. In this setting, we observed Dykstra gets stuck switching back and forth between projection onto $\rho'$ and projection onto occupancy measures $\Delta$ and we get division by zero exception. 
 In our experiments we observed similar outcome when using the penalty function $\K(\mu \bm{1}|\rho')$
 with high coefficient $\epsilon_1$ (close to 20). 
